\documentclass{article}
\usepackage{nips_2018}
\usepackage[usenames,dvipsnames]{xcolor}

\usepackage{graphicx,amsthm,amsmath, amssymb, natbib, enumerate}
\newtheorem{theorem}{Theorem}
\newtheorem{lemma}[theorem]{Lemma}
\newtheorem{corollary}[theorem]{Corollary}

\newtheorem{remark}{Remark}

\begin{document}
	
	%
	
	%
	
	\title{Statistical Optimality of Interpolated Nearest Neighbor Algorithms}
	\author{
		Yue ~Xing\\
		Department of Statistics\\
		Purdue University\\
		West Lafayette, Indiana, USA \\
		\texttt{xing49@purdue.edu} \\
		\And
		Qifan ~Song \\
		Department of Statistics \\
		Purdue University\\
		West Lafayette, Indiana, USA \\
		\texttt{qfsong@purdue.edu} \\
		\AND
		Guang ~Cheng \\
		Department of Statistics \\
		Purdue University \\
		West Lafayette, Indiana, USA \\
		\texttt{chengg@purdue.edu} \\
	}
	
	\maketitle
	\begin{abstract}
		In the era of deep learning, understanding over-fitting phenomenon becomes increasingly important. It is observed that carefully designed deep neural networks achieve small testing error even when the training error is close to zero. One possible explanation is that for many modern machine learning algorithms, over-fitting can greatly reduce the estimation bias, while not increasing the estimation variance too much.
		To illustrate the above idea, we prove that the proposed interpolated nearest neighbor algorithm achieves the minimax optimal rate in both regression and classification regimes, and observe that they are empirically better than the traditional $k$ nearest neighbor method in some cases. 
	\end{abstract}
	\section{Introduction}
	
	
	In deep learning, with the structure of neural networks getting more and more complicated, computer scientists proposed various approaches, such as Dropout, to handle the possible over-fitting issues (\cite{srivastava2014dropout,gal2016dropout,chen2014big,guo2016deep,lecun2015deep}). However, recent studies, for example \cite{zhang2016understanding}, demonstrate that deep neural networks have a small generalization error even when the training data is perfectly fitted.  Similar phenomenon of strong generalization performance for over-fitted models occurs in other modern machine learning algorithms as well, including kernel machines, boosting and random forests.
	
	Inspired by these observations, this work will investigate the statistical optimality of perfectly fitted (interpolated) models by nearest neighbor
	algorithms (NN). Specifically, given training data $\{x_i,y_i\}_{i=1}^n$, we study the regression estimator
	\[	 \widehat\eta(x) = \sum_{x_i\in N_k(x)}w_i y_i,	\]
	and its corresponding classifier 
	\[\widehat g(x)=1\{\widehat\eta(x)>1/2\},\]
	where $N_k(x)$ denotes the set of the nearest $k$ neighbors of $x$, and the weight $w_i$'s are designed in the way such that 
	$\widehat\eta(x_i)=y_i$.
	
	
	Traditional $k$-NN assigns $w_i=1/k$ and does not perfectly fit the data. Its asymptotic behavior and convergence rate have been well studied by many works (\cite{cover1968rates,wagner1971convergence,fritz1975distribution,schoenmakers2013optimal,sun2016stabilized,CD14}).
	
	\cite{belkin2018over-fitting,belkin2018does} designed an interpolated NN algorithm by a normalized polynomial weight function, and further extended this idea to Nadaraya-Watson kernel regression. \cite{belkin2018does} derived the regression optimal rate using Nadaraya-Watson kernel, and the results in  \citep{belkin2018over-fitting} are claimed to be optimal without proof. And our goal here is to design a new type of weighting scheme with proven optimal rates for both regression and classification objectives.	Specifically, the mean squared error (MSE) of $\widehat\eta(x)$ and the risk bound of the classifier $\widehat g(x)$ (under the  margin  condition  of \cite{tsybakov2004optimal}) are proven to be minimax optimal.


	
	Additionally, we provide an intuitive explanation on why the interpolated-NN can perform potentially better than the traditional $k$-NN. In fact, there is a bias-variance trade-off for the nearest neighbor methods: the traditional $k$-NN minimizes the variance to some extent, while interpolated-NN tries to reduce the bias. Although theoretically both of them attain the same optimal rate of MSE, our empirical studies demonstrate that interpolated-NN always yields better estimation and prediction. We conjecture that the superior performance of interpolated-NN over $k$-NN is due to a smaller multiplicative constant of convergence speed.
	

	\section{Interpolated-NN Algorithm and Model Assumptions}\label{alg}
	Let $\mathcal X\subset \mathbb{R}^d$ be the support of $X$ and $\mu$ be the probability measure of $X$ on $\mathcal{X}$. Define $\eta(x)=E(Y|X=x)$, where the response variable $Y$ can be either binary (classification problem) or continuous (regression problem). Given $n$ iid observations $\{x_i, y_i=y(x_i)\}_{i=1}^n$ and a test sample $x$, let $x_{(i)}$ denote the $i$th nearest neighbor of $x$ under $\mathcal L_2$ distance, and a weighted-NN algorithm predicts the mean response value as
	\begin{eqnarray*}
		\widehat\eta(x)=\sum_{i=1}^k w_iy(x_{(i)}),
	\end{eqnarray*}
	where $w_i=w_i(x_1,\dots,x_n)$ are some data-dependent nonnegative weights	satisfying $\sum_{i=1}^kw_i=1$.
	
	To induce an exact data interpolation, it is sufficient to require $w_i\rightarrow 1$ as $\|x_{(i)}-x\|\rightarrow0$. Following \cite{belkin2018over-fitting}, we construct the weight as:
	\begin{eqnarray*}
		w_i&=& \frac{\phi\big( \frac{\|x_{(i)}-x\|}{\|x_{(k+1)}-x\|} \big)}{\sum_{j=1}^k \phi\big( \frac{\|x_{(j)}-x\|}{\|x_{(k+1)}-x\|} \big)},
	\end{eqnarray*}
	for some positive function $\phi$ on [0,1] which satisfies $\lim_{t\rightarrow0}\phi(t)=\infty$, and if $\|x_{(i)}-x\|=0$, we define $w_i=1$. The denominator term $\|x_{(k+1)}-x\|$ is for normalization purpose. Note that conditional on $X_{(k+1)}$, $\|X_i-x\|/\|X_{(k+1)}-x\|$s are independent variables in [0,1] for all $X_i$ belonging to the $k$-neighborhood of $x$. The function $\phi$ plays a crucial role for the analysis of $\widehat\eta$ and $\widehat g$. For example, if $\phi(t)$ increases too fast as $t\rightarrow0$, the weighted average $\widehat\eta$ is always dominated by $y(x_{(1)})$. In this case, the variance of $\widehat\eta$ and $\widehat g$ will be too large. Thus the following condition on the choice of $\phi$ is needed:
	\begin{enumerate}
		\item[A.0] For any random variable $T\in[0,1]$ whose density is bounded, the moment generating function of $\phi(T)$ exists, i.e., there exist some $s>0$ and $M$, such that $\mathbb{E}(e^{s\phi(T)})< M$.
	\end{enumerate}
	Technically, this condition allows us to bound the un-normalized weight average $\sum_{i=1}^k\phi(\|x_{(i)}-x\|/\|x_{(k+1)}-x\|)y(x_{(i)})$ by exponential concentration inequalities.
	In general, any positive function $\phi$ satisfying $\phi(1/u) = O(\log(u))$ as $u\rightarrow \infty$ meets the above condition, and one typical example could be $\phi(t)= 1-c\log(t)$ for any constant $c>0$. \cite{belkin2018over-fitting} chose $\phi$ to be $\phi(t)=t^{-\kappa}$ for some $0<\kappa<d/2$, which unfortunately doesn't satisfy the above condition.
	Visual comparison among different choices of $\phi$ can be found in Figure \ref{fig:weight}. Some toy regression examples are demonstrated in the appendix to compare different weighting schemes of interpolated-NN with $k$-NN.

	\begin{figure}
		\centering
		\includegraphics[scale=0.9]{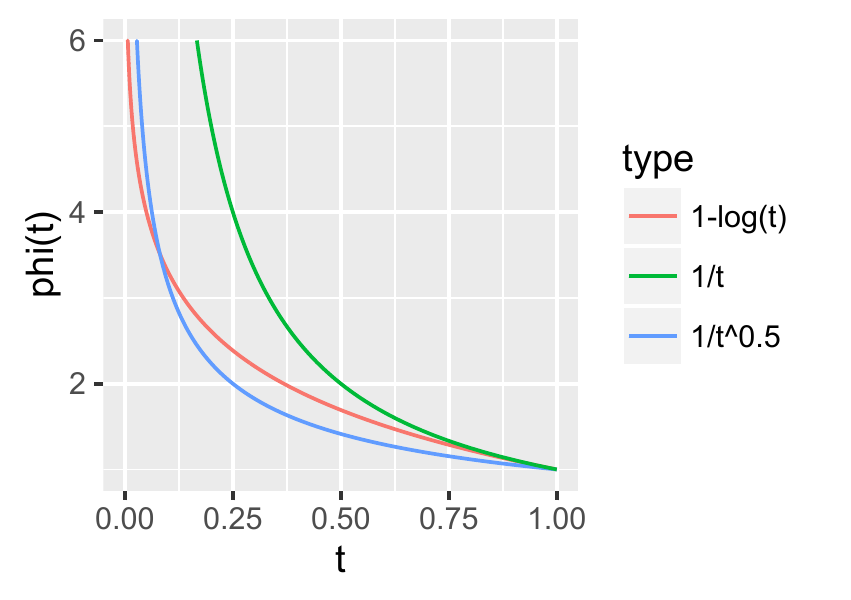}
		\caption{Interpolated Weighting Scheme}
		\label{fig:weight}
	\end{figure}

	Recall that the classifier is defined as $\widehat{g}(x):=1\left\{\widehat\eta(x)\ge 1/2\right\}$, with the Bayes classifier $g(x):=1\{\eta(x)\ge1/2\}$. Define the point-wise excess risk $R_{n,k}(x)-R^*(x)$, where
	\begin{eqnarray*}
		R_{n,k}(x)&=& P(Y\neq \widehat{g}(x)|X=x),\\
		R^*(x)&=&\min(\eta(x),1-\eta(x)).
	\end{eqnarray*}
	
	We next study the asymptotic rate of mean squared error of $\widehat\eta$ and expected excess risk for $\widehat g$ by imposing some regularity conditions. For any $d$-dimensional real-valued $x$ in $\mathcal{X}$, let $B(x,r)$ represent the closed ball with radius $r$ and center $x$.
	\begin{enumerate}
		\item[A.1] Finite variance: $\bar{\sigma}^2:=\sup\limits_{x\in\mathcal X}Var(Y|X=x)<\infty$.
		\item[A.2] Smoothness condition: $|\eta(x)-\eta(y)|\leq A\|y-x\|^{\alpha}$ for some $\alpha>0$.
		\item[A.3] Regularity condition: let $\lambda$ be the Lebesgue measure on $\mathbb{R}^d$, then there exists positive $(c_0,r_0)$ such that for any $x$ in the support $\mathcal{X}$,
		\begin{equation*}
		\lambda(\mathcal{X}\cap B(x,r))\geq c_0\lambda(B(x,r)),
		\end{equation*}
		for any $0<r\leq r_0$.
		\item[A.4] The density of $X$ is finite, and twice-continuously differentiable.
		\item[A.5] Margin condition:  $P(|\eta(x)-1/2|<t)\leq Bt^{\beta}$.
	\end{enumerate}
	Assumption (A.2) is commonly assumed in the literature. A larger value of $\alpha$ implies more accurate estimation of $\eta$ due to the minimax rate $O(n^{-2\alpha/(2\alpha+d)})$. 
	Assumption (A.3) was first introduced by \cite{audibert2007fast}, and it essentially ensures that for any $x\in\mathcal X$, all its $k$ nearest neighbors are sufficiently close to $x$ with high probability. 
	If $\mathcal{X}=\mathbb{R}^d$ or $\mathcal{X}$ is convex, this regularity condition is automatically satisfied. Assumption (A.5) is the so-called Tsybakov low noise condition (\cite{audibert2007fast, tsybakov2004optimal}), and this assumption is of interest mostly in the classification context. Under smoothness condition (A.2) and marginal condition (A.5), it is well known that the optimal rate for nonparametric regression and classification are $\mathbb{E}[(\widehat{\eta}(X)-\eta(X))^2]= O(n^{-2\alpha/(2\alpha+d)})$ and $\mathbb{E}[R_{n,p}(X)-R^*(X)]=O(n^{-\alpha(\beta+1)/(2\alpha+d)})$, respectively. Here, $\mathbb{E}$ means the expectation over all observed data $\{x_i,y_i\}_{i=1}^n$ and new observation $X\sim \mu$.
	
	The analysis of classification is very subtle especially when $\eta(x)$ is near $1/2$. Hence, we need to have the following partition over the space $\mathcal X$. 	Define $v_p(x)$ as the $p$-th quantile of $\|X-x\|$, and for the ball $B(x,r)$,
	\begin{equation*}
	\eta(B(x,r)):=\frac{\mathbb{E}\left(\phi\left(\frac{\|X_1-x\|}{\|X_{(k+1)}-x\|}\right)\eta(X_1)\bigg|\|X_{(k+1)}-x\|= r\right)}{\mathbb{E}\left(\phi\left(\frac{\|X_1-x\|}{\|X_{(k+1)}-x\|}\right)\bigg|\|X_{(k+1)}-x\|= r\right)}.
	\end{equation*}
	Also define
	\begin{align*}
	\mathcal{X}^+_{p,\Delta}=\{x\in\mathcal{X}|\eta(x)>\frac{1}{2},\eta(B(x,r))\geq &\frac{1}{2}+\Delta ,\forall r<v_p(x)\},\\
	\mathcal{X}^-_{p,\Delta}=\{x\in\mathcal{X}|\eta(x)<\frac{1}{2},\eta(B(x,r))\leq &\frac{1}{2}-\Delta, \forall r<v_p(x)\},
	\end{align*}
	with the decision boundary area:
	\begin{equation*}
	\partial_{p,\Delta}=\mathcal{X}\setminus(\mathcal{X}^+_{p,\Delta}\cup\mathcal{X}^-_{p,\Delta}).
	\end{equation*}
	
	\section{Main Results}
	In this section, we will present our asymptotic results for the interpolated-NN algorithm described in Section \ref{alg}. 
	Our main Theorems \ref{wnn} and \ref{class} state that the interpolated estimator $\widehat\eta$ and classifier $\widehat g$ are asymptotically rate-optimal in terms of MSE and excess risk, respectively. In other words, data interpolation or data over-fitting doesn't necessarily jeopardize the statistical performance of a learning algorithm, at least in the minimax sense. These theoretical findings are supported by our numerical experiments in Section~\ref{sec:num}.
	
	\subsection{Over-Fitting Is Sometimes Even Better}\label{knn_wnn_compare}
	Let us start from an intuitive comparison between the interpolated-NN and traditional $k$-NN. For any weighted-NN algorithm, the following bias-variance decomposition (\cite{belkin2018over-fitting}) holds (with $\alpha$ and $A$ in Assumption (A.2)):
	\begin{align}
	&\mathbb{E}\big\{ (\widehat{\eta}(x)-\eta(x))^2 \big\}\leq A^2\mathbb{E}\bigg[\sum_{i=1}^kW_i\|X_{(i)}-x\|^{\alpha}\bigg]^2+ \sum_{i=1}^k \mathbb{E}\bigg[W_i^2(Y(X_{(i)})-\eta(X_{(i)}))^2 \bigg].\label{mse}
	\end{align}
	This upper bound in (\ref{mse}) can be viewed as two parts: squared bias and variance, respectively, and $W_i$ denotes the random weight $w_i(X_1,\dots,X_n)$.
	The $k$-NN ($W_i\equiv 1/k$) can be interpreted as optimal weight choice which minimizes the variance term in the case that ${\rm Var}(Y_i)$ are constant, e.g., under regression setting that $Y_i=\eta(X_i)+\epsilon_i$ with iid error term $\epsilon_i$, or under classification setting that $Y_i\sim {\rm Bern}(\eta(X_i))$ with constant function $\eta$.
	On the other hand, interpolated-NN assigns larger weight for closer neighbor, which will result in a smaller value for the weighted average $\sum_{i=1}^kW_i\|X_{(i)}-x\|^{\alpha}$, i.e., the bias term gets smaller. Therefore, we argue that $k$-NN and interpolated-NN employ different strategies in reducing the upper bound of MSE. The former one emphasizes reducing the variance, while the latter one emphasizes reducing the bias. The above intuitive arguments are also well validated by our toy examples described in appendix.

	Later, we will prove that the proposed interpolated-NN method, although potentially enlarge the variance term, is still minimax rate-optimal for both regression and classification cases. 
	
	\begin{remark}
		In some literature on the asymptotic properties of $k$NN such as \cite{bailey1978note}, the number of neighbors, i.e., $k$, is assumed to be fixed, while $n$ diverges. Under this scenario, the upper bound in (\ref{mse}) becomes exactly the MSE, and all neighbors become samples at $x$. As a result, the best weighting scheme to minimize the variance of $k$ i.i.d. Bernoulli random variables is just $k$NN. 
	\end{remark}
	
	\subsection{Regression}
	
	As seen in the decomposition (\ref{mse}), the bias term involves $\|X_{(i)}-x\|$, which is the empirical quantile for $\|X-x\|$. Hence our first lemma, whose proof can be found in appendix, will study the asymptotic behavior of this empirical quantile.
	
	Let $F_x(\cdot)$, $\widehat F_x(\cdot)$ and $f_x(\cdot)$ be the c.d.f.,
	empirical c.d.f and p.d.f. of $\|X-x\|$, $v_p(x)$ and $\widehat {v}_p(x)$ be the $p$th quantile and $p$th empirical quantile of $\|X-x\|$, respectively.
	
	\begin{lemma}\label{lemma}
		Given any $x\in\mathcal{X}$, let  $p=k/n$ (thus $\widehat{v}_{p}(x)=\|X_{(k)}-x\|$), and 
		$$A_2(x)=F_x(\widehat{v}_{p}(x))+\widehat{F}_x(v_{p}(x))-\widehat{F}_x(\widehat{v}_{p}(x))-F_x(v_{p}(x)).$$
		Under Assumption (A.4), if $n^{-1}|\widehat{v}_{p}(x)-v_{p}(x)|\rightarrow0$, $1/(np)\rightarrow 0$, and $v_{p}(x)=O(p^{1/d})$ with $n^{-1/2}/p^{1/d-1/2}\rightarrow 0$, $f_x(v_{p}(x))=O(p^{1-1/d})$, then
		$$\mathbb{E}A_2^2(x)\rightarrow O(p^{1/2}/n^{3/2}). $$
	\end{lemma}

	The above lemma, together with the technical tool developed in \cite{sun2009general,sun2010asymptotic}, facilitates our analysis for the bias term in (\ref{mse}). As for the variance term in (\ref{mse}), we can bound it by the similar approach used by \cite{belkin2018over-fitting}. Together, it leads to the following theorem.
	
	\begin{theorem}[Rate-Optimality of Interpolated-NN Regression]\label{wnn}
		Denote $int(\mathcal{X})$ as the interior of $\mathcal{X}$. Under Assumption (A.0)-(A.4), if $\alpha\in(0,1]$, then for any fixed $x\in int(\mathcal{X})$, there exists constants $c_1\rightarrow 1^+$, $(\mu,x)$-dependent $c_2$, and a constant $c_3>0$, when $n\rightarrow \infty$,
		\begin{equation}\label{wnnresult}
		\begin{split}
		&\mathbb{E}\big\{ (\widehat{\eta}(x)-\eta(x))^2 \big\}
		\leq A^2 c_1 \bigg(\frac{k}{n} \bigg)^{2\alpha/d}+ c_2\bar{\sigma}^2\bigg( ke^{-c_3k}+\frac{1}{k} \bigg).
		\end{split}
		\end{equation}
		
		Therefore, taking $k\asymp (n^{2\alpha/(2\alpha+d)})$, $\mathbb{E}\big\{ (\widehat{\eta}(x)-\eta(x))^2 \big\}$ reaches the optimal rate of $O(n^{-2\alpha/(2\alpha+d)})$.
	\end{theorem}
	
	
	\begin{proof}[Proof]
		First, it is trivial to see that the bias-variance representation (\ref{mse}) can be further bounded by
		\begin{align*}
		\mathbb{E}\big[\widehat{\eta}(x)-\eta(x)\big]^2 \leq A^2\mathbb{E} \|X_{(k)}-x\|^{2\alpha}+ \bar{\sigma}^2k\mathbb{E}(W_1^2).
		\end{align*}
		
		Therefore the remaining task is to figure out the convergence rate 
		of $\mathbb{E} \|X_{(k)}-x\|^{2\alpha}$ and calculate $\mathbb{E}(W_1^2)$. 
		
		For bias term $\mathbb{E} \|X_{(k)}-x\|^{2\alpha}$, note that $\|X_{(k)}-x\|$ is the $p$(=$k/n$)th empirical quantile of $F_x$, i.e., $\widehat{v}_p(x)=\|X_{(k)}-x\|$ (for simplicity, we write $F_x$ as $F$, $f_x$ as $f$, $v_p(x)$ as $v_p$ and $\widehat{v}_p(x)$ as $\widehat{v}_p$ in what follows), and this quantile is called Value-at-Risk in the area of finance. From \cite{sun2009general}, we have
		\begin{eqnarray*}
			\widehat{v}_{p}&=& v_{p} + \frac{1}{f(v_{p})}\bigg( p-\frac{1}{n}\sum_{i=1}^n 1_{\{ r_i\leq v_{p} \}} \bigg) + \frac{1}{f(v_{p})} (A_1+A_2+A_3),
		\end{eqnarray*}
		where $A_1=O(f'(v_{p})(\widehat{v}_{p}-v_{p})^2)$, $A_2=F(\widehat{v}_{p})+\widehat{F}(v_{p})-\widehat{F}(\widehat{v}_{p})-F(v_{p})$, and $A_3=\widehat{F}(\widehat{v}_{p})-F(v_{p})$. Compared with $\widehat{v}_{p}-v_{p}$, $A_1/f(v_p)$ is a smaller order term. Since $p=k/n$, we always have $\widehat{F}(\widehat{v}_{p})=F(v_{p})=p$, hence $A_3\equiv 0$. 
		
		Within some neighborhood of $x$, $B(x,r)$ with $r\leq r_0$, both maximum and minimum density of $X$ are bounded. Denote them as $p_{\max}$ and $p_{\min}$. As $n\rightarrow\infty$, $v_p\leq r$, thus by Assumption (A.4), $v_{p}^dp_{\max}\geq k/n$ and c$v_{p}^dp_{\min}\leq k/n= F(v_{p})$ for some constant $c$. This implies that $v_{p}\geq (k/np_{\max})^{1/d}$, and 
		\begin{eqnarray*}
			f(v_{p})&=& O\bigg( \frac{k}{n} \bigg)^{1-1/d}.
		\end{eqnarray*}
		
		As a result, when $\alpha\in (0,1]$, define $r_i=\|X_i-x\|$, we have
		\begin{eqnarray*}
			&&\mathbb{E} \bigg( p-\frac{1}{n}\sum_{i=1}^n 1_{\{ r_i\leq v_{p} \}} \bigg)^{2\alpha} \leq\bigg[\mathbb{E} \bigg( p-\frac{1}{n}\sum_{i=1}^n 1_{\{ r_i\leq v_{p} \}} \bigg)^2\bigg]^{\alpha}
			= \bigg[\frac{k}{n^2}\bigg(1-\frac{k}{n}\bigg)\bigg]^{\alpha}\leq \bigg(\frac{k}{n^2}\bigg)^{\alpha}.
		\end{eqnarray*}
		Since $k\leq n$,\begin{eqnarray*}
			&&\frac{1}{f(v_{p})^{2\alpha}}\mathbb{E} \bigg( p-\frac{1}{n}\sum_{i=1}^n 1_{\{ r_i\leq v_{p} \}} \bigg)^{2\alpha}=O\bigg( \frac{k}{n} \bigg)^{2\alpha(1/d-1)} \mathbb{E} \bigg( p-\frac{1}{n}\sum_{i=1}^n 1_{\{ r_i\leq v_{p} \}} \bigg)^{2\alpha} =o\bigg(\frac{k}{n}\bigg)^{2\alpha/d}.
		\end{eqnarray*}
		In terms of $A_2$, based on Assumption (A.4) and Lemma \ref{lemma}, for $\alpha\in(0,1]$, 
		\begin{eqnarray*}
			\frac{1}{f(v_{p})^{2\alpha}}\mathbb{E}A_2^{2\alpha}&=&O\bigg( \frac{k}{n} \bigg)^{2\alpha(1/d-1)} \mathbb{E}A_2^{2\alpha}\\ &=&O\bigg(\frac{p^{-2\alpha+\alpha/2}}{n^{3\alpha/2}}\bigg)O\bigg( \frac{k}{n} \bigg)^{2\alpha/d} \\
			&=& O\bigg(\frac{p^{-3\alpha/2}}{n^{3\alpha/2}}\bigg)O\bigg( \frac{k}{n} \bigg)^{2\alpha/d}.
		\end{eqnarray*}
		As a result, when $pn\rightarrow\infty$, which is satisfied when taking $k\asymp(n^{2\alpha/(2\alpha+d)})$, we have
		\begin{eqnarray*}
			O\bigg(\frac{p^{-3\alpha/2}}{n^{3\alpha/2}}\bigg)=o(1),
		\end{eqnarray*}
		hence,
		\begin{eqnarray*}
			A^2\mathbb{E} \|X_{(k)}-x\|^{2\alpha}\leq A^2 c_1 \bigg(\frac{k}{n} \bigg)^{2\alpha/d},
		\end{eqnarray*}
		with $c_1\rightarrow 1^+$ in $n$. 
		
		For $\mathbb{E}W_1^2$, similar with \cite{belkin2018over-fitting}, we can bound it as
		\begin{eqnarray*}
			\mathbb{E}W_1^2&\leq& P(\|X_{(k+1)}-x\|>v_{2p})\\&&+\mathbb{E}[W_1^2|\|X_{(k+1)}-x\|<v_{2p}],
		\end{eqnarray*}
		where $v_{2p}$ is the $2k/n$th quantile of $\|X-x\|$.
		
		By the same argument in Claim A.6 of \cite{belkin2018over-fitting}, the probability $P(\|X_{(k+1)}-x\|>v_{2p})$ under Assumption (A.3) and (A.4) can be bounded by $O(\exp(-c_3k))$ for some $c_3>0$.
		
		On the other hand, by the arguments in Lemma 10 of \cite{CD14}, conditional on  $\|X_{(k+1)}-x\|$, for those $X_i$'s that belong to $x$'s $k$-neighborhood, $\|X_i-x\|/\|X_{(k+1)-x}\|$ are iid random variables in [0,1] . It follows that
		\begin{align*}
		\mathbb{E}[W_1^2|\|X_{(k+1)}-x\|<v_{2p}]
		=&\mathbb{E}\bigg[\frac{\left[\phi\left(\frac{\|X_1-x\|}{\|X_{(k+1)}-x\|} \right)\right]^2}{\left[\sum_{i=1}^k\phi\left(\frac{\|X_i-x\|}{\|X_{(k+1)}-x\|} \right)\right]^2}\bigg|\|X_{(k+1)}-x\|<v_{2p}\bigg]\\
		\leq& \frac{1}{k^2}\mathbb{E}\bigg[\phi\left(\frac{\|X_1-x\|}{\|X_{(k+1)}-x\|} \right)^2\bigg|\|X_{(k+1)}-x\|<v_{2p}\bigg].
		\end{align*}
		
		An $O(1/k)$ result can be directly obtained for $k\mathbb{E}[W_1^2|\|X_{(k+1)}-x\|<v_{2p}]$ through the existence of $\mathbb{E}e^{s\phi}$ (Assumption (A.0)).
	\end{proof}
	
	Theorem \ref{wnn} proves the point-wise MSE convergence result for $\widehat\eta$. By assuming that $\mathcal{X}$ is compact (which is also assumed in \cite{belkin2018over-fitting} and \cite{CD14}), it is not difficult to see that there exist constant $c_1$ and $c_2$ in (\ref{wnnresult}) such that  uniformly holds for all $x\in\mathcal{X}$.
	
	\begin{corollary}\label{coro}
		Under Assumption (A.0)-(A.4), if $\mathcal{X}$ is compact, then there exists constants $c_1\rightarrow 1^+$, $\mu$-dependent $c_2$, such that when $n\rightarrow \infty$,
		\begin{eqnarray*}
			\mathbb{E}\big\{ (\widehat{\eta}(X)-\eta(X))^2 \big\}\leq A^2 c_1 \bigg(\frac{k}{n} \bigg)^{2\alpha/d}+ c_2\bar{\sigma}^2\bigg( ke^{-c_3k}+\frac{1}{k} \bigg).
		\end{eqnarray*}
		
		Therefore, taking $k\asymp (n^{2\alpha/(2\alpha+d)})$, $\mathbb{E}\big\{ (\widehat{\eta}(X)-\eta(X))^2 \big\}$ reaches the optimal rate of $O(n^{-2\alpha/(2\alpha+d)})$.
	\end{corollary}


	\begin{remark}
		It is worth to mention that Assumption (A.0) is not necessary for Theorem \ref{wnn} and Corollary \ref{coro}. Our proof can be easily adapted to show that interpolated-NN regression estimator with polynomial $\phi(t)=t^{-\kappa}$ is also rate-optimal.\footnote{\cite{belkin2018over-fitting} claims the optimal rate of convergence of $\widehat{\eta}(x)$ with $\phi(t)=t^{-\kappa}$ based on heuristic argument on the order of $\mathbb{E}\|X_{(k)}-x\|$.} 
	\end{remark}
	\begin{remark}	
		Instead of the nearest neighbor method, one can alternatively consider kernel regression estimator $\sum_{\|x_i-x\|\leq h}w_iy(x_i)$ with interpolated weights (\cite{belkin2018does}). In such a case, the theoretical investigation of $p$th quantile of $F_x$ can be avoided, and if the bandwidth $h$ is of the same order of $v_{p}$, one can still obtain the optimal rate.
	\end{remark}


	\subsection{Classification}
	In this section, we investigate the theoretical properties of the interpolated-NN classifier $\widehat g$, and the next theorem establishes the statistical optimality of $\widehat g$ in terms of excess risk.
	
	\begin{theorem}[Rate-Optimality of Interpolated-NN Classification]\label{class}
		Under Assumption (A.0), (A.2)-(A.5), assume $\mathcal{X}$ is compact, then for any $\delta\in(0,1)$, taking $k\asymp(n^{2\alpha/(2\alpha+d)}(\log\delta)^{d/(2\alpha+d)})$, we have
		\begin{equation}\label{class1}
		P(g(X)\neq \widehat{g}(X))\leq \delta+C_0\bigg( \frac{\log(1/\delta)}{n} \bigg)^{\frac{\alpha\beta}{2\alpha+d}}
		\end{equation}
		for some constant $C_0>0$.
		
		Moreover, taking $k\asymp(n^{2\alpha/(2\alpha+d)})$,
		\begin{equation}\label{class2}
		\mathbb{E}(R_{n,p}(X)-R^*(X))\leq O(n^{-\frac{\alpha(\beta+1)}{2\alpha+d}}).
		\end{equation}
	\end{theorem}
	We remark that the above bounds in (\ref{class1}) and (\ref{class2}) are the same as those in \cite{CD14}.
	
	\begin{proof}[Proof]
		Let $p=2k/n$, following \cite{belkin2018over-fitting, CD14}, we obtain that for any fixed $\Delta\in(0,1)$,
		\begin{eqnarray*}
			P(g(X)\neq \widehat{g}(X))
			&\leq& \mu(\partial_{2p,\Delta})+P(E)\\
			&&+P\bigg( \sum_{i=1}^kW_i(Y(X_{(i)})-1/2)>0 \bigg| X\in\mathcal{X}_{2p,\Delta}^-{\rm and } E^c \bigg)\\
			&&+P\bigg( \sum_{i=1}^kW_i(Y(X_{(i)})-1/2)<0 \bigg| X\in\mathcal{X}_{2p,\Delta}^+{\rm and } E^c\bigg )\\
			&\leq& B\bigg(\Delta+AC_1\bigg(\frac{k}{n}\bigg)^{\alpha/d}\bigg)^{\beta}+B\exp(-k/4)\\
			&&+P\bigg( \sum_{i=1}^kW_i(Y(X_{(i)})-1/2)>0 \bigg| X\in\mathcal{X}_{2p,\Delta}^- {\rm and } E^c\bigg)\\
			&&+P\bigg( \sum_{i=1}^kW_i(Y(X_{(i)})-1/2)<0 \bigg| X\in\mathcal{X}_{2p,\Delta}^+ {\rm and } E^c\bigg),
		\end{eqnarray*}
		holds for some constants $A$, $B$ and $C_1$, where $r_{i}=\|X_{(i)}-X\|$, ${\rm E}=\{r_{k+1}>v_{2p}(X)\}$.
		Denote $Z_i(x)=\phi(\|X_i-x\|/r_{k+1})(Y(X_i)-1/2)$ for those $k$ nearest $X_i$'s.  By Assumption (A.0) of $\phi$ and the arguments in Theorem 5 of
		\cite{CD14}, given any fixed $x$, conditional on $r_{k+1}$, $Z_i(x)$ are iid sub-exponential variables. 
		Hence by Bernstein inequality and compact $\mathcal{X}$ assumption, for some $C_2, C_3, C_4>0$, conditional on $X\in\mathcal{X}_{2p,\Delta}^-$ and $E^c$,
		\begin{equation}\label{bern}
		\begin{split}
		P\bigg( \sum_{i=1}^kW_i(Y(X_{(i)})-1/2)>0\bigg)
		=P\bigg(\sum_{i=1}^k Z_i-\mathbb{E}Z_i>-k\mathbb{E}Z_1\bigg)
		\leq C_2\exp(-C_4 k\Delta^2 ).
		\end{split}
		\end{equation}
		The way to employ Bernstein inequality is postponed to appendix.
		
		Let $\delta:=2C_2\exp(-C_4 k\Delta^2)$, then replacing $\Delta$ with $\delta$, we have
		\begin{eqnarray*}
			P(g(X)\neq \widehat{g}(X))\leq O\bigg( \sqrt{\frac{\log 1/\delta}{k}}+\bigg(\frac{k}{n}\bigg)^{\alpha/d} \bigg)^{\beta} + B\exp(-c_3k)+\delta,
		\end{eqnarray*}
		then the proof of (\ref{class1}) is completed by taking $k\asymp(n^{2\alpha/(2\alpha+d)}(\log\delta)^{d/(2\alpha+d)})$.
		
		To prove (\ref{class2}), we follow the proof of \cite{CD14}. Without loss of generality assume $\eta(x)<1/2$. Define
		\begin{eqnarray*}
			\Delta_0&=& v_{2p}^\alpha =O(k/n)^{\alpha/d},\\
			\Delta(x)&=&|\eta(x)-1/2|,
		\end{eqnarray*}
		so that $\eta(B(x,r))\leq \Delta_0+\eta(x)=1/2-(\Delta(x)-\Delta_0)$ for $r<v_{2p}$, and $x$ is not in $\partial_{p,\Delta(x)-\Delta_0}$.
		
		To calculate $\mathbb{E}R_{n,k}(x)-R^*(x)$, from definition, we obtain
		\begin{eqnarray}\label{eqn:regret}
		\mathbb{E}R_{n,k}(x)-R^*(x)=|1-2\eta(x)|P(\widehat{g}(x)\neq g(x)).
		\end{eqnarray}
		In addition, from definition of $\partial_{p,\Delta}$, $\mathcal{X}_{p,\Delta}^+$, and $\mathcal{X}_{p,\Delta}^-$, if $x$ is not in $\partial_{p,\Delta}$, we also have
		\begin{eqnarray}\label{eqn:prob}
		P(\widehat{g}(x)\neq g(x))\leq P(r_{(k+1)}\geq v_{2p})+P(|\widehat{\eta}(x)-\eta(B(x,r_{(k+1)}))|\geq \Delta).
		\end{eqnarray}
		
		The derivations of (\ref{eqn:regret}) and (\ref{eqn:prob}) are postponed in appendix.
		
		As a result, when $\Delta(x)>\Delta_0$,
		\begin{equation*}
		\begin{split}
		\mathbb{E}R_{n,k}(x)-R^*(x)&\leq 2\Delta(x) \bigg[P(r_{(k+1)}>v_{2p})\\
		&+P\bigg( \sum_{i=1}^kW_i(Y(X_{(i)})-\eta(B(x,r_{(k+1)})))>\Delta(x)-\Delta_0  \bigg)\bigg].
		\end{split}
		\end{equation*}
		Applying Bernstein inequality (similarly to (\ref{bern})), we obtain
		\begin{equation}\label{eqn:berstein}
		\mathbb{E}R_{n,k}(x)-R^*(x)\leq \exp(-C_5k)+2\Delta(x)\exp(-c_4 k(\Delta(x)-\Delta_0)^2).  
		\end{equation}
		The details to derive (\ref{eqn:berstein}) are postponed to appendix.
		
		Define $\Delta_i=2^i\Delta_0$, for any $i_0=1,...$, the excess risk can be bounded by
		\begin{eqnarray*}
			&&\mathbb{E}R_{n,k}(X)-R^*(X)\\
			&=&\mathbb{E}(R_{n,k}(X)-R^*(X))1_{\{\Delta(X)\leq\Delta_{i_0}\}}+\mathbb{E}(R_{n,k}(X)-R^*(X))1_{\{\Delta(X)>\Delta_{i_0}\}}\\
			&\leq& \mathbb{E}(2\Delta(X)1_{\{\Delta(X)\leq \Delta_{i_0}\}})+\exp(-C_5k)
			+4\mathbb{E}[\Delta(x)\exp(-k(\Delta(x)-\Delta_0)^2)1_{\{\Delta(X)> \Delta_{i_0}\}}]\\
			&\leq& 2\Delta_{i_0} P(\Delta(X)\leq \Delta_{i_0})+\exp(-C_5k)+4\mathbb{E}[\Delta(x)\exp(-k(\Delta(x)-\Delta_0)^2)1_{\{\Delta(X)> \Delta_{i_0}\}}],
		\end{eqnarray*}	
		while
		\begin{eqnarray*}
			&&\mathbb{E}[\Delta(x)\exp(-k(\Delta(x)-\Delta_0)^2)1_{\{\Delta_i<\Delta(X)\leq \Delta_{i+1}\}}]\\
			&\leq&\Delta_{i+1}\exp(-C_6k(\Delta_i-\Delta_0)^2)P(\Delta(X)\leq \Delta_{i+1})\\
			&\leq& C_7\Delta^{\beta+1}_{i+1}\exp(-C_6k(\Delta_i-\Delta_0)^2).
		\end{eqnarray*}	
		Taking \begin{equation*}
		i_0=\max\bigg(1,\left\lceil  \log_2\sqrt{  \frac{2\beta+4}{k\Delta_0^2}}\right\rceil\bigg),
		\end{equation*}
		then for $i>i_0$, we have
		\begin{equation}\label{fail}
		\frac{\Delta_{i+1}^{\beta+1}\exp(-C_6k(\Delta_i-\Delta_0)^2)}{\Delta_{i}^{\beta+1}\exp(-C_6k(\Delta_{i-1}-\Delta_0)^2)}\leq 1/2.
		\end{equation}
		Therefore, the sum of the excess risk for $i>i_0$ can be bounded, where
		\begin{eqnarray*}
			\mathbb{E}[\Delta(x)\exp(-k(\Delta(x)-\Delta_0)^2)1_{\{\Delta(X)> \Delta_{i_0}\}}]
			&\leq& \sum_{i\geq i_0}C_7\Delta_{i+1}^{\beta+1}\exp(-C_6k(\Delta_i-\Delta_0)^2)\\
			&\leq&O(\Delta_{i_0}^{\beta+1}).
		\end{eqnarray*}	
		Finally, recall the definition of $i_0$, we have
		\begin{eqnarray*}
			\mathbb{E}R_{n,k}(X)-R^*(X)\leq \max\bigg(O\bigg(\frac{k}{n}\bigg)^{\frac{\alpha(\beta+1)}{d}},O\bigg( \frac{1}{k} \bigg)^{\frac{\beta+1}{2}}\bigg).
		\end{eqnarray*}
		The proof is completed after taking $k\asymp(n^{2\alpha/(2\alpha+d)})$.
	\end{proof}
	
	\begin{remark}
		It is worth to mention that, unlike Theorem \ref{wnn} and Corollary \ref{coro} which are asymptotic results, Theorem \ref{class} presents an non-asymptotic bound for misclassification rate and excessive risk.
	\end{remark}
	
	\begin{remark}
		Comparing with the choice $\phi(t)=t^{-\kappa}$ proposed by \cite{belkin2018over-fitting}, our choice (e.g., $\phi(t)=1-\log(t)$) leads to a sharp bound for classification error rate. Technically, this is due to the fact that slowly increasing $\phi(\cdot)$ allows us to use exponential concentration inequality.
	\end{remark}
	
	\section{Numerical Experiments}\label{sec:num}
	In the section, we demonstrate several numerical studies to compare the performance of interpolated-NN with $\phi(t)=1-2\log(t)$ and traditional $k$-NN, for both regression and classification problems.
	\subsection{Regression}\label{knn_exp}
	We have two simulation setups. In the first setup, $X$ follows $\text{Unif}[-3,3]^{10}$. For $y$, we have
	\begin{eqnarray}
	y&=&\psi(x)+\epsilon,\\
	\psi(x)&=& \frac{N(x,1,I_{10})}{N(x,1,I_{10})+N(x,0,I_{10})},
	\end{eqnarray} 
	where $\epsilon\sim t(\text{df}=5)$ and $N(x,\mu,\sigma)$ represents the density function of $N(\mu,\sigma)$ at $x$. 
	
	For each pair of $(k,n)$, we sampled $1000$ testing data points to estimate MSE, and repeated 30 times. For each $n$, we tried $k=1,...,n/2$. Based on the average of the 30 repetitions, we selected the minimum average MSE over the choices of $k$ and recorded its corresponding square bias. From Figure \ref{knn_wnn}, for both MSE and bias, the  interpolated-NN and $k$-NN share the same rate (the decreasing slopes for both are similar), but the former is constantly better.
	
	\begin{figure}
		\begin{center}
			\includegraphics[scale=0.7]{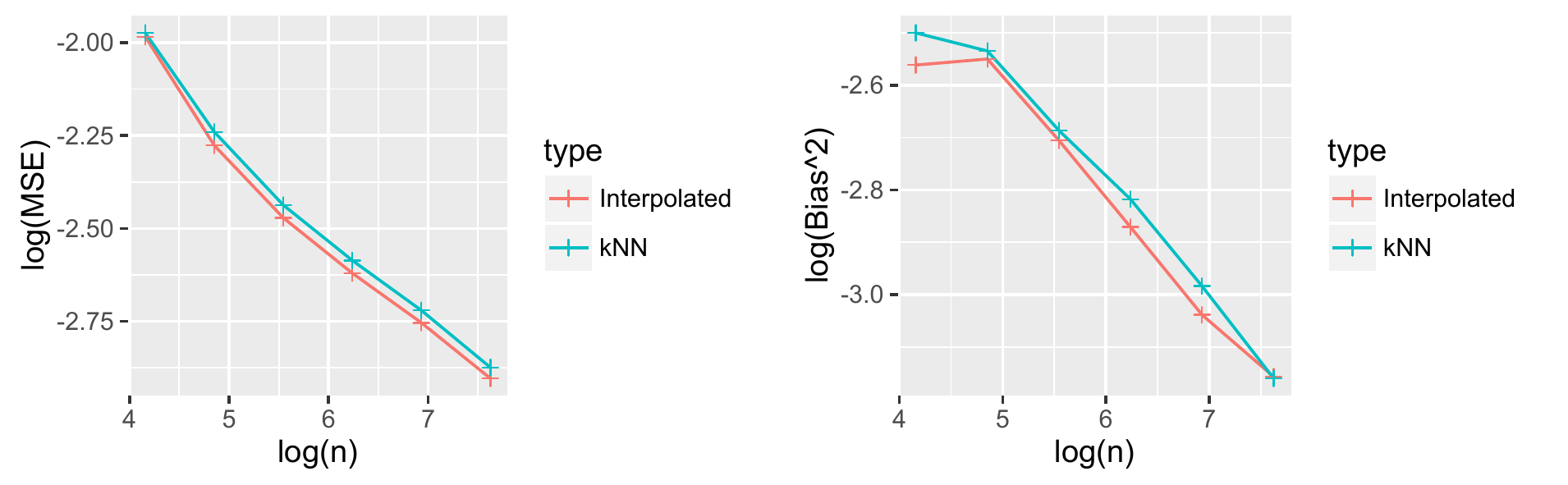}
		\end{center}
		\caption{Optimal MSE and the Corresponding Bias for Regression Model 1}\label{knn_wnn}
	\end{figure}

	In the second experiment, $X$ follows $N(0,I_5)$. The response $y$ is defined as
	\begin{equation*}
	y=(\sum_{i=1}^5x_i)^2+\epsilon,
	\end{equation*}
	with $\epsilon\sim N(0,1)$. The phenomenon demonstrated in Figure \ref{knn_wnn2} is similar as that in Figure \ref{knn_wnn}, although the support of $X$ is not compact in this setup.
	\begin{figure}
		\begin{center}
			\includegraphics[scale=0.7]{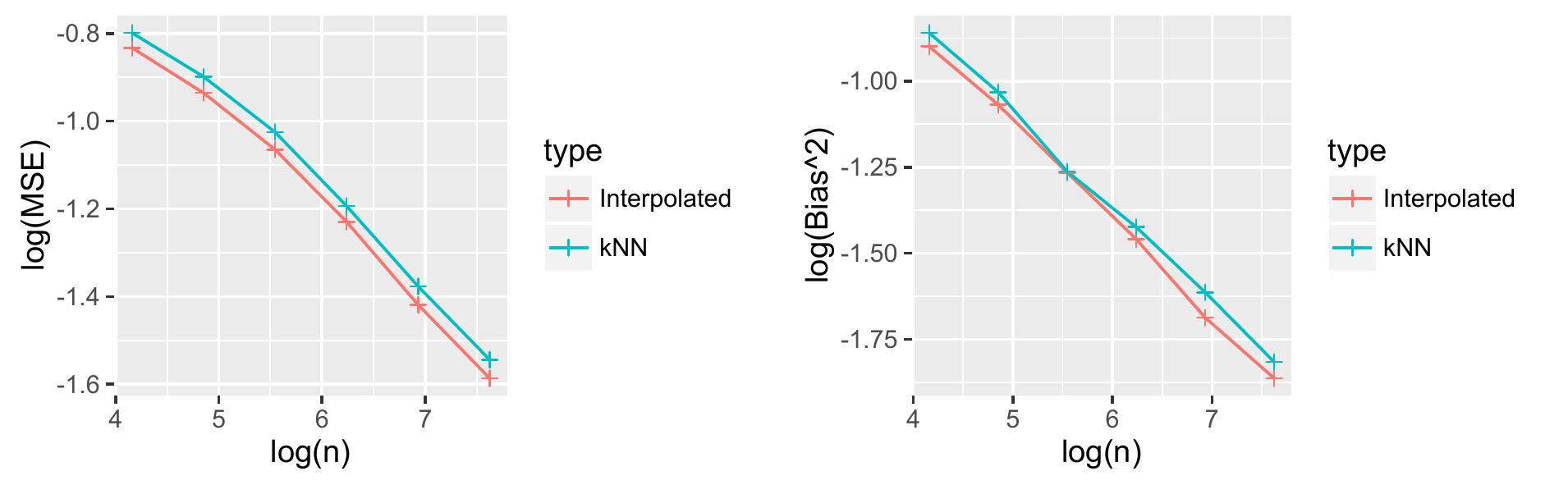}
		\end{center}
		\caption{Optimal MSE and the Corresponding Bias for Regression Model 2}\label{knn_wnn2}
	\end{figure}
	
	It is worth to mention, in both simulation, the optimal $k$ values selected by interpolated-NN and $k$-NN doesn't have much difference.
	
	\subsection{Classification}
	For the simulation setup of classification problem, we consider that the two classes follow $N(0,I_5)$ and $N(\gamma\textbf{1},I_5)$ with $\gamma$ = 0.1, 0.2, 0.5, 0.7, 1.0, 1.5. Training and testing samples are both generated from the mixture of these two classes with equal probability. Such an equal probability mixture represents the worst (most difficult) scenario. For each pair of $n$ and $\gamma$, we tried $k=1,...,n/2$ for 30 times with 1000 test samples, and use the excess risk to determine the optimal $k$. From Figure \ref{class_fig}, we observe that the interpolated-NN always has a smaller classification error (with a similar trend, though) than $k$-NN in most pairs of $n$ and $\gamma$. Moreover, we record the optimal $k$ for each $n$. As shown in Figure \ref{k_fig}, interpolated-NN and $k$-NN have a similar pattern on the choice of optimal $k$ across different settings of $n$ and $\gamma$.
	
	\begin{figure}[h]
		\begin{center}
			\includegraphics[scale=0.75]{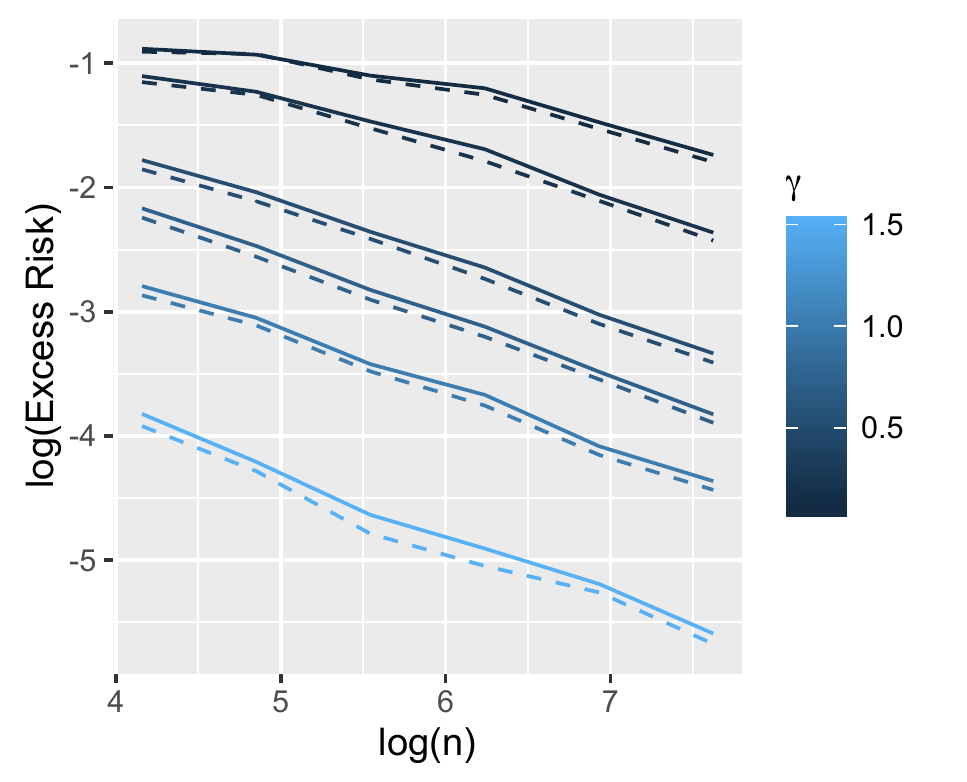}
		\end{center}
		\caption{Optimal Excess Risk for Classification Model, Solid Line for $k$-NN, Dashed Line for interpolated-NN}\label{class_fig}
	\end{figure}

	\begin{figure}[h]
		\begin{center}
			\includegraphics[scale=0.73]{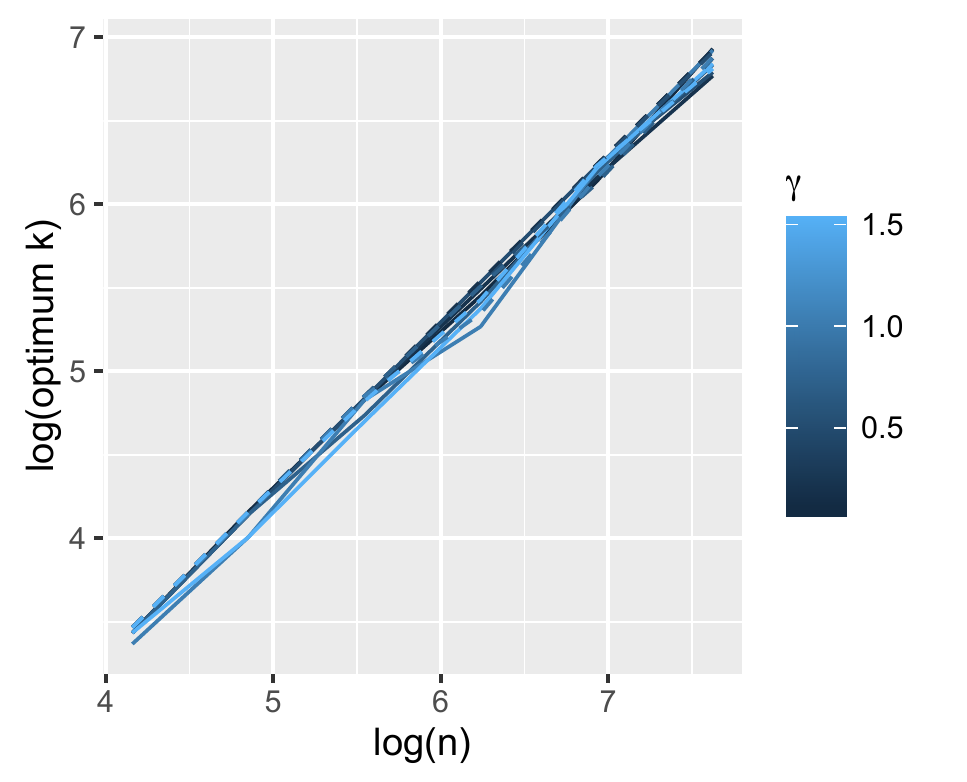}
		\end{center}
		\caption{Optimal $k$ for Classification Model}\label{k_fig}
	\end{figure}
	
	\begin{remark}
		The simulation results presented in above two sections clearly show that interpolated-NN and $k$-NN share the same order of convergence rate, as their decreasing treads are parallel. However, the interpolated-NN is always better than the $k$-NN. This observation strongly suggests that the convergence rate of interpolated-NN has a smaller multiplicative constant than the $k$-NN.
	\end{remark}
	
	\subsection{Real Data}
	In this section, we examine the empirical performance of the interpolated-NN in real data. The data HTRU2 is from \cite{lyon2016fifty} with sample size around $18,000$ and $8$ continuous attributes. The dataset is first normalized, and 2,000
	randomly selected samples are reserved for testing.

	Instead of comparing regression performance through MSE, we compare the classification error on testing data in this real data example since the true Bayes risk is unknown. It can be seen from Figure \ref{real1} that the interpolated-NN is always better than $k$-NN regardless of the choice of $k$.
	
	\begin{figure}
		\centering
		\includegraphics[scale=0.8]{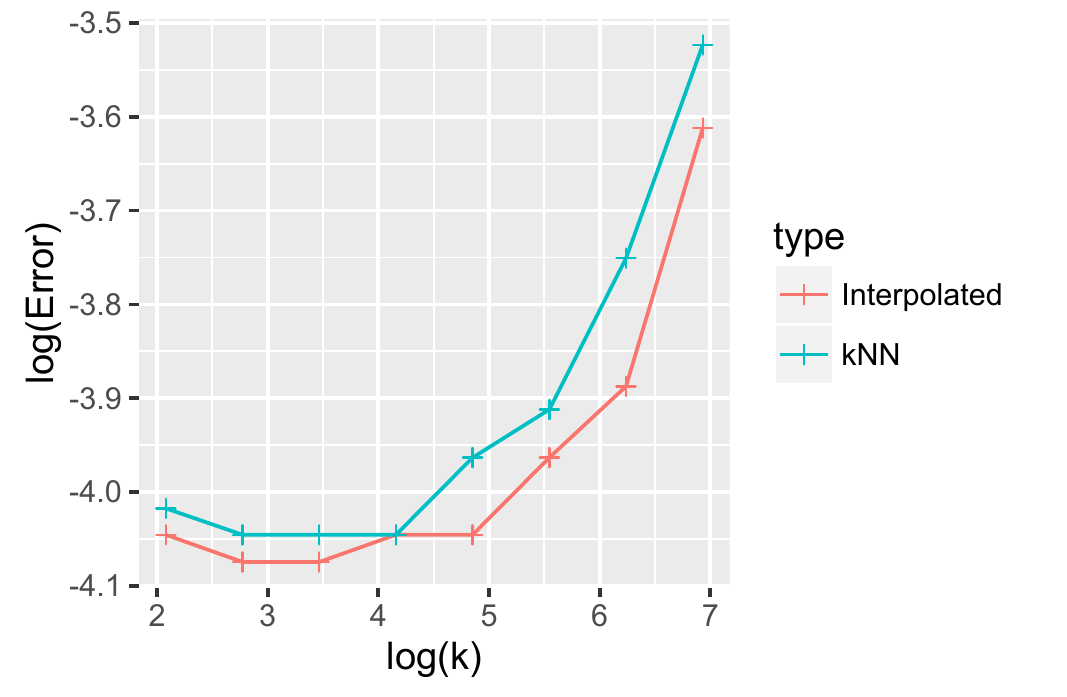}
		\caption{Performance in HTRU2}
		\label{real1}
	\end{figure}

	\section{Conclusion and Discussion}
	In this paper, we firstly provide some insights on why sometimes an over-fitting weighting scheme can even beat traditional $k$-NN: the interpolated weights greatly reduce the estimation bias comparing with tradition $k$-NN. We then show the optimal convergence rates of interpolated-NN for both regression and classification. Even though the weighting scheme causes over-fitting, interpolated-NN can still obtain the optimal rate, as long as the weights are carefully designed.
	
	In the end, we would like to point out a few promising future directions: firstly, as we mentioned, in most cases of our simulations, the interpolated-NN performs better than $k$-NN. This motivates us to conjecture that the interpolated-NN may have a smaller multiplicative constant in convergence rate, which deserves further investigation. 
	
	Our simulations indicate that the interpolated-NN performs very well even when the compactness assumption of $\mathcal{X}$ is violated. Therefore, it will be of interests to extend our current work to  unbounded $\mathcal{X}$ case. Especially, we notice that in \cite{doring2017rate}, the compact assumption on $\mathcal{X}$ can be relaxed for traditional $k$-NN algorithm. Similar results may hold for interpolated-NN as well.
	
	Finally, our results are based on $\mathcal{L}_2$ metric. It is not hard to generalize it to $\mathcal{L}_p$, but it remains to study whether the statistical optimality is still valid for some more general metric spaces, such as Riemannian manifolds.

	\bibliographystyle{asa}
	\bibliography{VaRHDIS}

\begin{thebibliography}{20}
\newcommand{\enquote}[1]{``#1''}
\expandafter\ifx\csname natexlab\endcsname\relax\def\natexlab#1{#1}\fi

\bibitem[{Audibert and Tsybakov(2007)}]{audibert2007fast}
Audibert, J.-Y. and Tsybakov, A.~B. (2007), \enquote{Fast learning rates for
  plug-in classifiers,} \textit{The Annals of statistics}, 35, 608--633.

\bibitem[{Belkin et~al.(2018{\natexlab{a}})Belkin, Hsu, and
  Mitra}]{belkin2018over-fitting}
Belkin, M., Hsu, D., and Mitra, P. (2018{\natexlab{a}}), \enquote{Overfitting
  or perfect fitting? Risk bounds for classification and regression rules that
  interpolate,} \textit{arXiv preprint arXiv:1806.05161}.

\bibitem[{Belkin et~al.(2018{\natexlab{b}})Belkin, Rakhlin, and
  Tsybakov}]{belkin2018does}
Belkin, M., Rakhlin, A., and Tsybakov, A.~B. (2018{\natexlab{b}}),
  \enquote{Does data interpolation contradict statistical optimality?}
  \textit{arXiv preprint arXiv:1806.09471}.

\bibitem[{Chaudhuri and Dasgupta(2014)}]{CD14}
Chaudhuri, K. and Dasgupta, S. (2014), \enquote{Rates of convergence for
  nearest neighbor classification,} in \textit{Advances in Neural Information
  Processing Systems}, pp. 3437--3445.

\bibitem[{Chen and Lin(2014)}]{chen2014big}
Chen, X.-W. and Lin, X. (2014), \enquote{Big data deep learning: challenges and
  perspectives,} \textit{IEEE access}, 2, 514--525.

\bibitem[{Cover(1968)}]{cover1968rates}
Cover, T.~M. (1968), \enquote{Rates of convergence for nearest neighbor
  procedures,} in \textit{Proceedings of the Hawaii International Conference on
  Systems Sciences}, pp. 413--415.

\bibitem[{D{\"o}ring et~al.(2017)D{\"o}ring, Gy{\"o}rfi, and
  Walk}]{doring2017rate}
D{\"o}ring, M., Gy{\"o}rfi, L., and Walk, H. (2017), \enquote{Rate of
  convergence of k-nearest-neighbor classification rule,} \textit{The Journal
  of Machine Learning Research}, 18, 8485--8500.

\bibitem[{Fritz(1975)}]{fritz1975distribution}
Fritz, J. (1975), \enquote{Distribution-free exponential error bound for
  nearest neighbor pattern classification,} \textit{IEEE Transactions on
  Information Theory}, 21, 552--557.

\bibitem[{Gal and Ghahramani(2016)}]{gal2016dropout}
Gal, Y. and Ghahramani, Z. (2016), \enquote{Dropout as a Bayesian
  approximation: Representing model uncertainty in deep learning,} in
  \textit{international conference on machine learning}, pp. 1050--1059.

\bibitem[{Guo et~al.(2016)Guo, Liu, Oerlemans, Lao, Wu, and Lew}]{guo2016deep}
Guo, Y., Liu, Y., Oerlemans, A., Lao, S., Wu, S., and Lew, M.~S. (2016),
  \enquote{Deep learning for visual understanding: A review,}
  \textit{Neurocomputing}, 187, 27--48.

\bibitem[{LeCun et~al.(2015)LeCun, Bengio, and Hinton}]{lecun2015deep}
LeCun, Y., Bengio, Y., and Hinton, G. (2015), \enquote{Deep learning,}
  \textit{nature}, 521, 436.

\bibitem[{Lyon et~al.(2016)Lyon, Stappers, Cooper, Brooke, and
  Knowles}]{lyon2016fifty}
Lyon, R.~J., Stappers, B., Cooper, S., Brooke, J., and Knowles, J. (2016),
  \enquote{Fifty years of pulsar candidate selection: from simple filters to a
  new principled real-time classification approach,} \textit{Monthly Notices of
  the Royal Astronomical Society}, 459, 1104--1123.

\bibitem[{Schoenmakers et~al.(2013)Schoenmakers, Zhang, and
  Huang}]{schoenmakers2013optimal}
Schoenmakers, J., Zhang, J., and Huang, J. (2013), \enquote{Optimal dual
  martingales, their analysis, and application to new algorithms for Bermudan
  products,} \textit{SIAM Journal on Financial Mathematics}, 4, 86--116.

\bibitem[{Srivastava et~al.(2014)Srivastava, Hinton, Krizhevsky, Sutskever, and
  Salakhutdinov}]{srivastava2014dropout}
Srivastava, N., Hinton, G., Krizhevsky, A., Sutskever, I., and Salakhutdinov,
  R. (2014), \enquote{Dropout: a simple way to prevent neural networks from
  overfitting,} \textit{The Journal of Machine Learning Research}, 15,
  1929--1958.

\bibitem[{Sun and Hong(2009)}]{sun2009general}
Sun, L. and Hong, L.~J. (2009), \enquote{A general framework of importance
  sampling for value-at-risk and conditional value-at-risk,} in \textit{Winter
  Simulation Conference}, pp. 415--422.

\bibitem[{Sun and Hong(2010)}]{sun2010asymptotic}
--- (2010), \enquote{Asymptotic representations for importance-sampling
  estimators of value-at-risk and conditional value-at-risk,}
  \textit{Operations Research Letters}, 38, 246--251.

\bibitem[{Sun et~al.(2016)Sun, Qiao, and Cheng}]{sun2016stabilized}
Sun, W.~W., Qiao, X., and Cheng, G. (2016), \enquote{Stabilized nearest
  neighbor classifier and its statistical properties,} \textit{Journal of the
  American Statistical Association}, 111, 1254--1265.

\bibitem[{Tsybakov(2004)}]{tsybakov2004optimal}
Tsybakov, A.~B. (2004), \enquote{Optimal aggregation of classifiers in
  statistical learning,} \textit{Annals of Statistics}, 135--166.

\bibitem[{Wagner(1971)}]{wagner1971convergence}
Wagner, T. (1971), \enquote{Convergence of the nearest neighbor rule,}
  \textit{IEEE Transactions on Information Theory}, 17, 566--571.

\bibitem[{Zhang et~al.(2016)Zhang, Bengio, Hardt, Recht, and
  Vinyals}]{zhang2016understanding}
Zhang, C., Bengio, S., Hardt, M., Recht, B., and Vinyals, O. (2016),
  \enquote{Understanding deep learning requires rethinking generalization,}
  \textit{arXiv preprint arXiv:1611.03530}.

\end{thebibliography}
	\appendix

		%
		
		%
		
		\newpage
		\section{Toy Examples}
		Several toy examples are conducted to demonstrate nearest neighbor algorithm with interpolated weights.
		
		In these examples, we take 30 one dimensional training samples $x_i=-5,-4,...,25$, and generate three choices of response (1) $y=0*x+\epsilon$, (2) $y=x^2+0*\epsilon$, and (3) $y=(x-10)^2/8+5*\epsilon$ where $\epsilon\sim N(0,1)$. In other words, the mean function $\eta(x)$ are (1) $\eta(x)\equiv 0$, (2) $\eta(x)=x^2$, and (3) $\eta(x)=(x-10)^2/8$. The number of neighbors $k$ equals 10. 
		
		Three different weighting schemes are considered (1) $\phi(t)=1-\log(t)$, (2) $\phi(t)=1/k$, and (3) $\phi(t)=t^{-1}$.
		Note that the first and third choices are interpolated weight, and the second choice is simply the traditional $k$-NN. In particular, the first $\phi(\cdot)$ satisfies Assumption (A.0). The four line "true" refers to $\eta(x)$.
		
		Based on this setting, we plot the regression estimator $\widehat{\eta}(x)$ in Figure \ref{toy}. In order to remove the boundary effect, we only plot the  $\widehat{\eta}(x)$ within range $(0,20)$. Note that for the second case, in order to make the difference more clear and recognizable, we only plot $\widehat\eta(\cdot)$ for $x$ between 10 and 15.
		
		\begin{figure}
			\begin{center}
				
				\includegraphics[scale=1]{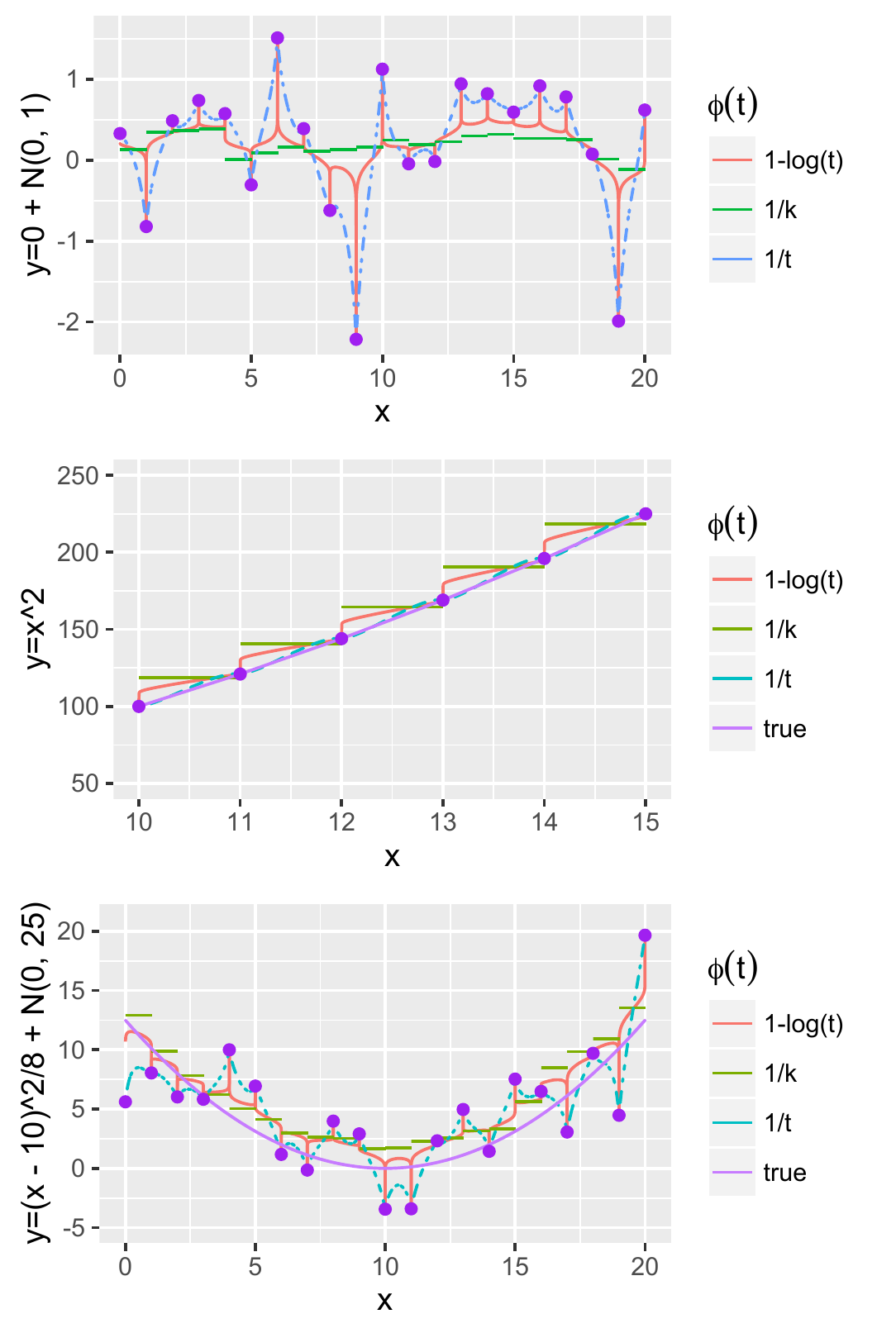}
				\caption{Three Trials on the Toy Simulation: Upper: $\eta\equiv0$; Middle:  $\eta(x)=x^2$;
					Lower: $\eta(x)=(x-10)^{2}/8$.}
				\label{toy}
			\end{center}
		\end{figure}
		
		There are several insights we can obtain from the results of this toy example.
		
		First of all, interpolated weight does ensure data interpolation. As $x$ gets closer to the some observed $x_i$, the estimator $\widehat\eta(x)$ is forced towards $y_i$. As a consequence, $\widehat\eta(x)$ is spiky for interpolated-NN. On contrast, the $k$-NN estimator is much more smooth.
		
		Secondly, different weighting schemes lead to different balance between bias and variance of $\widehat\eta$.
		In the first setting that $\eta\equiv 0$, any weight-NN algorithm is unbiased, hence it corresponds to the extreme situation that bias is 0;  the second model is noiseless, hence corresponds to the opposite extreme situation that variance is 0. 
		In the no-bias setting, $k$-NN performs the best, and interpolated-NN estimators fluctuate a lot. In the noiseless case, $k$-NN has the largest bias, and $\phi(t)=t^{-1}$ leads to smallest bias. These observations are consistent to our arguments in Section 3.1 in the main text, i.e., $k$-NN tries to minimize the variance of nearest neighbor estimator as much as possible, while interpolated-NN tries to minimize the estimation bias. 
		For the comparison between different interpolated weighting scheme, we comment that the faster $\phi(t)$ increases to infinite as $t\rightarrow0$, the smaller bias it will yield. Thus $\phi(t)=t^{-1}$ leads to smaller bias then $\phi(t)=1-\log(t)$, at the expense of larger estimation variance.
		
		For the third model $\eta(x)=(x-10)^2/8$, it involves both noise and bias. From Figure \ref{toy}, the estimation of $\phi(t)=1/k$ is always above the true line, i.e., high bias, due to the convexity of the $\eta$. For interpolated-NN, their estimators are more rugged, but fluctuate along the true $\eta(x)$. For this case, it is difficult to claim which one is the best, but the trade-off phenomenon between bias and variance is still clear to see.
		
		In conclusion, a fast increasing $\phi$ leads to smaller bias and larger variance, and non-interpolated weights such as $k$-NN leads to larger bias but smaller variance.
		
		\section{Proof of Lemma 1}

		\begin{lemma}\label{lemma}
			Given any $x\in\mathcal{X}$, let  $p=k/n$ (thus $\widehat{v}_{p}(x)=\|X_{(k)}-x\|$), and 
			$$A_2(x)=F_x(\widehat{v}_{p}(x))+\widehat{F}_x(v_{p}(x))-\widehat{F}_x(\widehat{v}_{p}(x))-F_x(v_{p}(x)).$$
			Under Assumption (A.4), if $n|\widehat{v}_{p}(x)-v_{p}(x)|$ and $pn\rightarrow\infty$, and $v_{p}(x)=O(p^{1/d})$ with $n^{1/2}p^{1/d-1/2}\rightarrow\infty$, $f(v_{p}(x))=O(p^{1-1/d})$, then
			$$\mathbb{E}A_2^2(x)\rightarrow O(p^{1/2}/n^{3/2}). $$
		\end{lemma}
		\begin{proof}[Proof]
			For simplicity, we write $F_x$ as $F$, $f_x$ as $f$, $v_p(x)$ as $v_p$ and $\widehat{v}_p(x)$ as $\widehat{v}_p$ in this proof. Denote $$A_2(a,b)=F(a)-\widehat{F}(a)-(F(b)-\widehat{F}(b)),$$ then for any $a,b\in\mathbb{R}^d$ with $|a-b|$ is larger than order $O(1/n)$, under Assumption (A.4), it follows that
			\begin{eqnarray}\label{lemma_order}
			\mathbb{E}A_2^{2m}(a,b)=O\bigg( \frac{(a-b)^m(f(a)+f(b))^m}{n^m}\bigg),
			\end{eqnarray}
			for any positive integer $m$. 
			
			To illustrate (\ref{lemma_order}), for example, when $m=4$, after expanding all the terms, the terms involving odd moments will be zero. The non-zero terms left are 
			\begin{equation}
			\label{haha1}\bigg[\mathbb{E}1_{\{ \|X-x\|\in(a,b) \}}-(F(a)-F(b)))^2\bigg]^2
			\end{equation}
			and 
			\begin{equation}\label{haha2}
			\mathbb{E}\bigg[1_{\{ \|X-x\|\in(a,b) \}}-(F(a)-F(b))\bigg]^4.
			\end{equation}
			For (\ref{haha1}), there are $O(n^2)$ terms with this form, while for (\ref{haha2}), there are only $O(n)$ terms. Since $a-b$ is larger than $O(1/n)$, the dominant part becomes (\ref{haha1}). Therefore we can obtain (\ref{lemma_order}) through further approximating $F(a)-F(b)$ as $[(a-b)(f(a)+f(b))]/2$. 
			
			Since $\widehat{F}(\widehat{v}_{p})=F(v_{p})=p$, we have
			\begin{eqnarray*}
				\widehat{F}(v_{p})-\widehat{F}(\widehat{v}_{p})=\widehat{F}(v_{p})-F(v_{p}),
			\end{eqnarray*}
			hence if $|\widehat{v}_{p}-v_{p}|$ and $p$ is larger than $O(1/n)$,
			\begin{eqnarray*}
				\mathbb{E}\big[\widehat{F}(v_{p})-\widehat{F}(\widehat{v}_{p})\big]^{2m}=\mathbb{E}\big[\widehat{F}(v_{p})-F(v_{p})\big]^{2m}= O\bigg(\frac{p^m}{n^m}\bigg).
			\end{eqnarray*}
			As a result, for $a,b>0$, without loss of generality, assume $\widehat{v}_{p}>v_{p}$,
			\begin{eqnarray*}
				&&\mathbb{E}\big[ F(v_{p})-{F}(\widehat{v}_{p}) \big]^{2m}\\
				&\leq& P( \widehat{F}(\widehat{v}_{p})-\widehat{F}(v_{p})>b)+P(\widehat{F}(v_{p}+a)-\widehat{F}({v}_{p})<b)\\&&+\mathbb{E}\bigg[ (F(v_{p}+a)-{F}({v}_{p}))^{2m}1_{ \widehat{F}(\widehat{v}_{p})\leq\widehat{F}(v_{p}+a)  }\bigg]\\
				&\leq&  P( \widehat{F}(\widehat{v}_{p})-\widehat{F}(v_{p})>b)+P(\widehat{F}(v_{p}+a)-\widehat{F}({v}_{p})<b)+\mathbb{E}\big[ F(v_{p}+a)-{F}({v}_{p})\big]^{2m}\\
				&\leq& O\bigg( \frac{p^{m_2}}{b^{2m_2}n^{m_2}} \bigg)+O\big( a^{2m}(f(a+v_{p})+f(v_{p}))^{2m} \big)+P(\widehat{F}(v_{p}+a)-\widehat{F}({v}_{p})<b)\\
				&=& O\bigg( \frac{p^{m_2}}{b^{2m_2}n^{m_2}} \bigg)+O\big( a^{2m}(f(a+v_{p})+f(v_{p}))^{2m} \big)\\&&+P(F(a+v_{p})-F(v_{p})<b+A_2(v_{p},v_{p}+a))\\
				&\leq& O\bigg( \frac{p^{m_2}}{b^{2m_2}n^{m_2}} \bigg)+O\big( a^{2m}(f(a+v_{p})+f(v_{p}))^{2m} \big)+ O\bigg(\frac{b^{2m_1}}{a^{2m_1}(f(v_{p})+f(a+v_{p}))^{2m_1}}\bigg)\\&&+O\bigg(\frac{1}{a^{m_1}(f(v_{p})+f(a+v_{p}))^{m_1}}\frac{1}{n^{m_1}}\bigg).
			\end{eqnarray*}
			When $n\rightarrow\infty$, taking $m_2\gg m_1\gg m$, we obtain
			\begin{eqnarray*}
				\mathbb{E}\big[ F(v_{p})-{F}(\widehat{v}_{p}) \big]^{2m}\leq O\bigg(\frac{p^{m}}{n^{m}}\bigg).
			\end{eqnarray*}
			
			One can check that $a=o(v_{p})$, hence $f(v_{p})\approx f(v_{p}+a)$.
			
			As a result, $\widehat{v}_{p}$ falls in the interval of $v_{p}\pm c$ with probability tending to 1 for some $c>0$. We first assume that $c=o(v_{p})$ hence $f(v_{p}+c)=O(f(v_{p}))$. Rewrite $f=f(v_{p}+c)f_(v_{p})$ for simplicity, then
			\begin{eqnarray*}
				P(|\widehat{v}_{p}-v_{p}|>c)\leq O\bigg(\frac{p^{m}}{n^{m}}\frac{1}{c^{2m}f^{2m}}\bigg).
			\end{eqnarray*}
			Therefore,
			\begin{eqnarray*}
				\mathbb{E}A_2(\widehat{v}_{p},v_{p})^2\leq \mathbb{E}A_2^2(c+{v}_{p},v_{p})+O\bigg(\frac{p^{m}}{n^{m}}\frac{1}{c^{2m}f^{2m}}\bigg)
				=O\bigg(\frac{cf}{n} \bigg)+O\bigg(\frac{p^{m}}{n^{m}}\frac{1}{c^{2m}f^{2m}}\bigg).
			\end{eqnarray*}
			The optimal rate becomes $O(p^{1/2}/n^{3/2})$ when $n\gg m\rightarrow \infty$. Note that $cf/n=O(p^{1/2}/n^{3/2})$ implies $c=O(p^{1/d}p^{-1/2}/n^{-1/2})=o(p^{1/d})=o(v_{p})$. Since $p^{1/d-1/2}$ is of larger  than $O(n^{-1/2})$, $c$ is also larger than  $O(1/n)$. 
		\end{proof}
		\section{Some Detailed Proofs of Theorem 4}
		\subsection{Bernstein Concentration Inequality} Under assumption (A.0), for $\phi(\|X_{(1)}-x\|)$, 
		denote
		\begin{equation*}
		\mathbb{E}\exp\left(\lambda\phi(\|X_{(1)}-x\|-\lambda\mathbb{E}\phi(\|X_{(1)}-x\|)\right):=e^{\psi(\lambda)},
		\end{equation*}
		then we can use Taylor's expansion on $\psi(\cdot)$ and obtain
		\begin{equation*}
		\mathbb{E}\exp\left(\lambda\phi(\|X_{(1)}-x\|-\lambda\mathbb{E}\phi(\|X_{(1)}-x\|)\right)=e^{ \sigma^2\lambda^2/2+O(\lambda^3) }.
		\end{equation*}
		As a result, there exists ($b$,$\nu$) such that for any $|\lambda|<1/b$,
		\begin{equation*}
		\mathbb{E}\exp\left(\lambda\phi(\|X_{(1)}-x\|)\right)\leq e^{\nu^2\lambda^2/2}\exp\left(\lambda\mathbb{E}\phi(\|X_{(1)}-x\|)\right).
		\end{equation*}
		Therefore, one can adopt Bernstein inequality for sub-exponential random variable to obtain
		\begin{equation*}
		\mathbb{E}\left[\sum_{i=1}^k\phi(\|X_{(i)}-x\|)-k\mathbb{E}\phi(\|X_{(1)}-x\|)\geq t\right]\leq\begin{cases}
		e^{-t^2/2\nu^2k}\qquad& 0\leq t\leq k\nu^2/b\\
		e^{-t/2\nu k}\qquad& t>k\nu^2/b
		\end{cases}.
		\end{equation*}
		
		In Theorem 4, we use Bernstein inequality for some $Z_i$ instead of $\phi$. Since $Z_i$ preserves the finite MGF property of $\phi$, the concentration bound is valid with some other $(b,\nu)$.
		\subsection{Expression of Regret}
		To show
		\begin{eqnarray*}
			\mathbb{E}R_{n,k}(x)-R^*(x)=|1-2\eta(x)|P(\widehat{g}(x)\neq g(x)),
		\end{eqnarray*}
		assume $\eta(x)<1/2$, 
		\begin{eqnarray*}
			&&P(\widehat{g}(x)\neq Y|X=x)-\eta(x)\\&=& \eta(x)P(\widehat{g}(x)=0|X=x) + (1-\eta(x))P(\widehat{g}(x)=1|X=x) -\eta(x)\\
			&=&\eta(x)P(\widehat{g}(x)=g(x)|X=x)+ (1-\eta(x))P(\widehat{g}(x)\neq g(x)|X=x)-\eta(x)\\
			&=& \eta(x)-\eta(x)P(\widehat{g}(x)\neq g(x)|X=x)+ (1-\eta(x))P(\widehat{g}(x)\neq g(x)|X=x)-\eta(x)\\
			&=& (1-2\eta(x))P(\widehat{g}(x)\neq g(x)|X=x),
		\end{eqnarray*}
		similarly, when $\eta(x)>1/2$, we have
		\begin{eqnarray*}
			P(\widehat{g}(x)\neq Y|X=x)-1+\eta(x)= (2\eta(x)-1)P(\widehat{g}(x)\neq g(x)|X=x).
		\end{eqnarray*}	
		As a result, 
		\begin{eqnarray*}
			\mathbb{E}R_{n,k}(x)-R^*(x)=|1-2\eta(x)|P(\widehat{g}(x)\neq g(x)).
		\end{eqnarray*}
		\subsection{Decomposition of Miss-classification Event}
		To obtain\begin{eqnarray*}
		P(\widehat{g}(x)\neq g(x))\leq P(r_{(k+1)}\geq v_{2p})+P(|\widehat{\eta}(x)-\eta(B(x,r_{(k+1)}))|\geq \Delta),
		\end{eqnarray*}
		from definition of $\partial_{p,\Delta}$, $\mathcal{X}_{p,\Delta}^+$, and $\mathcal{X}_{p,\Delta}^-$, the event of $g(x)\neq \widehat{g}(x)$ can be covered as:
		\begin{eqnarray*}
			1_{\{ g(x)\neq \widehat{g}_{k,n,\gamma}(x) \}}&\leq& 1_{\{ x\in \partial_{p,\Delta}\}}\\
			&&+1_{\{ r_{(k+1)}\geq v_{2p} \}}\\
			&&+1_{\{ |\widehat{\eta}(x)-\eta(B(x,r_{(k+1)}))|\geq \Delta \}}.
		\end{eqnarray*}
		Assume $\eta(B(x,r_{(k+1)})>1/2$ and $x\in\mathcal{X}_{p,\Delta}^+$, when  $\widehat{\eta}(x)<1/2$, $$\eta(B(x,r_{(k+1)}))-\widehat{\eta}(x) >\eta(B(x,r_{(k+1)})-1/2\geq \Delta.$$
		A similar result can be obtained when $x\in\mathcal{X}_{p,\Delta}^-$.
		
		Therefore the decomposition leads to the upper bound in (\ref{eqn:prob}).
		\subsection{Adopting Bernstein Inequality}
		Our aim is to obtain the bound
		\begin{eqnarray*}
			\mathbb{E}R_{n,k}(x)-R^*(x)&\leq& 2\Delta(x) \bigg[P(r_{(k+1)}>v_{2p})\\
			&&+P\bigg( \sum_{i=1}^kW_i(Y(X_{(i)})-\eta(B(x,r_{(k+1)}))>\Delta(x)-\Delta_0  \bigg)\bigg]\\ &\leq&\exp(-k/8)+2\Delta(x)\exp(-C_6k(\Delta(x)-\Delta_0)^2).
		\end{eqnarray*}
		Denote $\phi_i=\phi\left(\frac{\|X_{(i)}-x\|}{\|X_{(k+1)}-x\|}\right)$ for simplicity, then
		fixing $r_{(k+1)}$, since $\mathbb{E}\phi_1Y(X_{(i)})=\eta(B(x,r_{(k+1)}))$, we have
		\begin{eqnarray*}
			&& \Delta(x)P\bigg( \sum_{i=1}^kW_iY(X_{(i)})-\eta(B(x,r_{(k+1)}))>\Delta(x)-\Delta_0 \bigg)\\
			&=& \Delta(x) P\bigg( \sum_{i=1}^k\phi_iY(X_{(i)})>(\eta(B(x,r_{(k+1)}))+\Delta(x)-\Delta_0)\sum_{i=1}^k\phi_i \bigg)\\
			&=& \Delta(x) P\bigg( \sum_{i=1}^k\phi_iY(X_{(i)})-k\mathbb{E}\phi_1\eta(X_{(1)})-[\eta(B(x,r_{(k+1)}))+\Delta(x)-\Delta_0]\sum_{i=1}^k\left(\phi_i-\mathbb{E}\phi_1\right)\\&&\qquad\qquad\qquad\qquad\qquad\qquad>k[\eta(B(x,r_{(k+1)}))+\Delta(x)-\Delta_0]\mathbb{E}\phi_1-k\mathbb{E}\phi_1\eta(X_{(1)}) \bigg)\\
			&=& \Delta(x) P\bigg\{ \sum_{i=1}^k\phi_i(Y(X_{(i)})-\eta(B(x,r_{(k+1)})))-(\Delta(x)-\Delta_0)\left(\sum_{i=1}^k\phi_i-k\mathbb{E}\phi_1\right)\\&&\qquad\qquad\qquad\qquad\qquad\qquad>k(\Delta(x)-\Delta_0)\mathbb{E}\phi_1\eta(X_{(1)})\bigg\}.
		\end{eqnarray*}
		Note that
		\begin{eqnarray*}
			&&\mathbb{E}\sum_{i=1}^k\phi_i(Y(X_{i})-\eta(B(x,r_{(k+1)})))=0,\\
			&&\mathbb{E}(\Delta(x)-\Delta_0)\left(\sum_{i=1}^k\phi_i-k\mathbb{E}\phi_1\right)=0,
		\end{eqnarray*}
		hence we denote $$Z_i(x)=\phi_i(Y(X_{(i)})-\eta(B(x,r_{(k+1)})))-(\Delta(x)-\Delta_0)\phi_i+(\Delta(x)-\Delta_0)\mathbb{E}\phi_1$$for simplicity. The choice of $\phi$ satisfies A.0, which indicates that through adopting Bernstein inequality, fixing $r_{(k+1)}$, there exists some constant $c_4>0$, such that
		\begin{eqnarray*}
			&&\Delta(x)P\bigg( \sum_{i=1}^kW_i(Y(X_{(i)})-\eta(B(x,r_{(k+1)}))>\Delta(x)-\Delta_0  \bigg)\\&=&\Delta(x)P\left(\sum_{i=1}^k Z_i>k(\Delta(x)-\Delta_0)\mathbb{E}\phi_1\eta(X_{(1)})\right)\\&\leq&\Delta(x)\exp(-C_6k(\Delta(x)-\Delta_0)^2).
		\end{eqnarray*}
		
\end{document}